\documentclass[final]{alt2026} 

\jmlrpages{}      
\jmlrvolume{}     
\jmlrissue{}      
\jmlryear{}       

\makeatletter
\renewcommand*{\@jmlrproceedings}{}
\makeatother

\makeatother
\title[Uniform Convergence Beyond Glivenko-Cantelli]{Uniform Convergence Beyond Glivenko-Cantelli}
\usepackage{times}
\usepackage{lipsum}
\usepackage{algorithm}
\usepackage{algpseudocode}
\usepackage{bbm}
\usepackage{mleftright}
\usepackage{tikz}
\usepackage{mathrsfs}
\usepackage{enumitem}

\altauthor{\Name{Tanmay Devale} \Email{tdevale@purdue.edu}\\
 \Name{Pramith Devulapalli} \Email{pdevulap@purdue.edu}\\
 \Name{Steve Hanneke} \Email{steve.hanneke@gmail.com}\\
 \addr Purdue University}

\newcommand{\QQ}{\mathcal{Q}}
\newcommand{\prob}{\mathbb{P}}
\newcommand{\PP}{\mathcal{P}}
\newcommand{\NN}{\mathbb{N}}
\newcommand{\expect}{\mathbb{E}}
\newcommand{\A}{\mathcal{A}}

\newcommand{\parof}[1]{\left( #1 \right)}
\newcommand{\sqof}[1]{\left[ #1 \right]}
\newcommand{\curlof}[1]{\left\{ #1 \right\}}
\newcommand{\modu}[1]{\left|#1\right|}

\newcommand{\cov}[3]{\left#1 #2 \right #3}

\newcommand{\eps}{\varepsilon}
\newcommand{\1}{\mathbbm{1}}
\newcommand{\norm}[1]{\left\lVert#1\right\rVert}

\begin{document}

\maketitle

\begin{abstract}
  We characterize conditions under which collections of distributions on $\{0,1\}^\NN$ admit uniform estimation of their mean. Prior work from \cite{vc1971} has focused on uniform convergence using the empirical mean estimator, leading to the principle known as $P-$ Glivenko-Cantelli. We extend this framework by moving beyond the empirical mean estimator and introducing Uniform Mean Estimability, also called UME-learnability, which captures when a collection permits uniform mean estimation by any arbitrary estimator. We work on the space created by the mean vectors of the collection of distributions. For each distribution, the mean vector records the expected value in each coordinate. We show that separability of the mean vectors is a sufficient condition for UME-learnability. However, we show that separability of the mean vectors is not necessary for UME-learnability by constructing a collection of distributions whose mean vectors are non-separable yet UME-learnable using techniques fundamentally different from those used in our separability-based analysis. Finally, we establish that countable unions of UME-learnable collections are also UME-learnable, solving the conjecture posed in  \cite{cohen2025empiricalmeanminimaxoptimal}.
\end{abstract}


\section{Introduction}
The seminal work of \cite{vc1971} establishes that for any binary function class $\mathcal{F}$, finite VC dimension guarantees uniform convergence independent of the distribution. However, in settings where $\mathcal{F}$ admits infinite VC dimension, uniform convergence can still hold for some distributions provided specific properties are satisfied; in such cases, we say $\mathcal{F}$ satisfies the $P-$Glivenko-Cantelli property, which is described as 
\begin{equation}
    \label{eqn:pgc}
    \expect_{S\sim P^n} \sqof{\sup_{f \in \mathcal{F}} \modu{\mathbb{P}_nf-Pf} } \xrightarrow{n \rightarrow \infty} 0, 
\end{equation}
where $S$ is a set of $n$ i.i.d. data points sampled from $P$, $\mathbb{P}_nf$ is the empirical mean of $f$ computed using $S$, and $Pf$ is the true mean of $f$. The result of \cite{vc1971} characterizes the distributions that satisfy the $P-$ Glivenko-Cantelli property for any binary function class $\mathcal{F}$. \\
The work of \cite{cohen2025empiricalmeanminimaxoptimal} raised the direction of going beyond the empirical mean estimator and posited a conjecture, which we presently resolve and prove a stronger result in this paper. In this work, we consider collections of distributions defined over a countable function class $\mathcal{F}$ that admit uniform mean estimation even if it fails to satisfy the $P-$ Glivenko-Cantelli property. Our goal is to identify conditions under which the following holds:
\begin{equation}
    \label{eqn:ume}
    \expect_{S \sim P^n} \sqof{\sup_{f \in \mathcal{F}} \modu{\mathscr{P}_nf-Pf} } \xrightarrow{n \rightarrow \infty} 0, 
\end{equation}
where $\mathscr{P}_nf$ is an arbitrary estimator that uses $S$ to estimate $Pf$ the true mean of $f$.\\
For countable, binary-valued concept classes, $P-$Glivenko–Cantelli can also be equivalently expressed as a distribution on $\{0,1\}^{\mathbb N}$ where the $j^{\mathrm{th}}$ coordinate of the distribution is equivalent to the output of the $j^{\mathrm{th}}$ function. The works of \cite{cohen2023localglivenkocantelli,cohen2025empiricalmeanminimaxoptimal} can be understood through this lens, and it is the one we adopt. This alternative formulation allows us to work directly with distributions over $\{ 0, 1\}^\NN$, thereby abstracting away the explicit choice of a concept class.
Formally speaking, equation \eqref{eqn:pgc} can be rewritten as equation \eqref{eqn:lgc} below in the following way:
\begin{equation}
    \label{eqn:lgc}
    \expect_{S \sim \mu^n} \sqof{\sup_{j \in \NN} \modu{\hat q_j - q_j}} \;=\; \expect_{S \sim \mu^n} \norm{\hat q - q}_\infty
\end{equation}
where $\mu$ is a distribution on $\{0, 1 \}^{\NN}$, $S$ is a sample of size $n$, and $\hat q_j$ and $q_j$ are the empirical mean and true mean respectively for the $j^{\mathrm{th}}$ function in the countable function class.\\
In this paper, we characterize properties of a collection of distributions $\QQ$ on $\{0,1\}^{\NN}$ that ensure \eqref{eqn:lgc} converges to $0$ as $n\to\infty$, and we address the broader question of replacing the empirical estimator $\hat q$ with an arbitrary estimator $\tilde q$.
More technically, we define Uniform Mean Estimation (UME) learnability\footnote{The term ``uniform'' here refers to uniformity over indices, not uniformity over distributions (though see Appendix~\ref{app:uniform-learnability}).} in the following way: there exists an algorithm $\mathcal{A}$ such that for any ground truth distribution $\mu^* \in \mathcal{Q}$, given $n$ data points, it produces an estimate $\tilde{q}$ of the true mean vector $q \in [0, 1]^{\NN}$ satisfying
\begin{equation}
    \label{eqn:learningsetting}
    \expect_{S \sim \mu^n} \norm{\A(S)-q}_\infty = \expect_{S \sim \mu^n} \norm{\tilde q-q}_\infty \xrightarrow{n \rightarrow \infty} 0.
\end{equation}
The motivation for this framework can be seen in the limitations of the empirical mean as observed in \cite{cohen2023localglivenkocantelli}. Consider the collection of distributions $\QQ = \{\mu\}$ where $\mu$ is a product measure and Mean$(\mu) = \parof{\frac{1}{2},\frac{1}2,\ldots}$. As the coordinates of $\mu$ are independent, we obtain $X_j \sim \text{Bernoulli}\parof{\frac12}$. An obvious algorithm is to return the mean of the only distribution in the collection, which is $\parof{\frac{1}{2},\frac{1}{2},\ldots}$. But even for such a trivial collection, the empirical mean estimator fails. The probability of obtaining all $0$s or all $1$s for $n$ data points at a particular coordinate is positive. There are infinitely many coordinates at which this could occur; hence, it will almost surely occur. Hence, the empirical mean estimator cannot be used to estimate this collection of distributions even though it consists of only one distribution.
\subsection*{Our Contributions}
\begin{itemize}
  \item \textbf{Separability implies Learnability:} We study when collections of distributions $\QQ$ are UME-learnable and prove that if the collection of mean vectors corresponding to $\QQ$ is separable, then $\QQ$ is UME-learnable (Theorem \ref{thm:separable}).
  \item \textbf{Closed under Unions:} We prove that any countable union of UME-learnable collections of distributions is also UME-learnable (Theorem~\ref{thm:countableunions}). In particular, this resolves the conjecture of \cite{cohen2025empiricalmeanminimaxoptimal} for the union of two families and extends it to countably many families.
  \item \textbf{Beyond Separability:} A natural question we tackle is whether separability is necessary for UME-learnability. We illustrate that it is not a necessary condition by constructing a collection $\QQ$ whose space of mean vectors is non-separable, yet it is UME-learnable (Proposition~\ref{Prop:Tree}). Moreover, our construction utilizes techniques fundamentally different from those used in Theorem~\ref{thm:separable}, which may be of independent interest.
\end{itemize}
\section{Related Works}
\textbf{Classical Empirical Process Theory}
Our work stems from classical empirical process theory, which aims to characterize the conditions under which the empirical estimator converges to the true mean uniformly over a class of functions. For binary functions, \cite{vc1971} provides necessary and sufficient conditions that are independent of the underlying distribution guaranteed by the finiteness of a combinatorial quantity known as the VC dimension. They also characterize distributions for which the empirical mean estimator is a uniform estimator for the true mean using VC entropy. The subsequent work \cite{vapnik1981uniform} obtains that sub-exponential growth of the empirical covering numbers is also necessary and sufficient for uniform convergence. Modern expositions and refinements of these results can be found in \cite{Vapnik2006,vanderVaart2023}.\\
\textbf{Product Measures on $\{ 0, 1\}^{\mathbb{N}}$}
\cite{cohen2023localglivenkocantelli} study product measures on $\{0,1\}^\NN$ that are uniformly estimable by the empirical mean estimator. They identify the largest collection of estimable product measures, which they call the $LGC$ class. They show that $LGC$ consists of exactly those distributions whose mean vectors $q$ satisfy $T(q) = \sup_{j \in \NN} \frac{\log(j+1)}{\log(1/q_j)}$ is finite. We are motivated by their framework in developing the notion of UME-learnability over countable function classes, but we differ in two respects: we drop the reliance on product measures and assume any arbitrary mean estimator. \\
\textbf{Dependent Coordinates and Arbitrary Estimators}
The more recent works of \cite{blanchard2024correlatedbinomialprocess} and \cite{cohen2025empiricalmeanminimaxoptimal} extend the analysis from \cite{cohen2023localglivenkocantelli}. \cite{blanchard2024correlatedbinomialprocess} drop the assumption of product measures and 
analyze necessary and sufficient conditions of uniform convergence of the empirical mean estimator to the true mean when different coordinates can be correlated. 
On the other hand, \cite{cohen2025empiricalmeanminimaxoptimal} explores other arbitrary mean estimators besides the empirical mean estimator while keeping their attention focused on product measures.
They derive specific conditions that a product measure must satisfy for it to be UME-learnable by the empirical mean estimator. They also provide non-trivial extensions of the $LGC$ class when certain restrictions are relaxed.\\
\textbf{Infinite-Dimensional Exponential Families}
\cite{densityEstimationInfDimExpFam} studies an infinite dimensional exponential family of densities and constructs an estimator that can effectively predict the unknown density. Our setting is fundamentally different because we do not assume a common reference measure on which to define a density. We work with collections of measures defined on the space $\{ 0, 1\}^{\mathbb{N}}$ without assuming any common dominating measure. As a result, our analysis falls outside the scope of \cite{densityEstimationInfDimExpFam}.
\subsection{Notation}
For any $k \in \NN$, we write $[k] = \{i \in \NN: i \le k\}$. All logarithms are base $e$ unless otherwise specified. The floor and ceiling functions are denoted by $\lfloor t \rfloor$ and $\lceil t \rceil$ for $t \in \mathbb{R}$ mapping $t$ to the nearest integer below or above, respectively. Unspecified constants $c,c'$ may change from line to line.\\
We denote our collection of distributions with $\QQ$. For any distribution $\mu$ with mean $q$, a data point is denoted by $X$, indicates $X \sim \mu$. The realization denoted by $X^{(i)}_j$ refers to the $j^{\mathrm{th}}$ coordinate of the $i^{\mathrm{th}}$ data point. We overload our notation and use the superscript to enumerate a countable collection of distributions. For example, if $\QQ$ is a countable collection then $\mu^i_j$ denotes the $j^{\mathrm{th}}$ coordinate of the $i^{\mathrm{th}}$ distribution in the collection. The same convention applies to means. The measure-theoretic nuisances of defining distributions on $\{0,1\}^\NN$ have been addressed in \cite{cohen2023localglivenkocantelli}.\\
Some of the results in this paper and the literature are specific to product measures, hence we say $\mu = \text{Prod}(q)$ if $X \sim \mu$ is equivalent to $X_j \sim \text{Bernoulli}(q_j)$ for every $j \in \NN$. We say a collection of distributions is a collection of product measures if for every $\mu \in \QQ, \mu = \text{Prod}(q)$ for some $q \in [0,1]^\NN$.
\section{Definitions and Main Results}
For any distribution $\mu$ on $\{0,1\}^{\mathbb{N}}$, letting $X \sim \mu$, Mean$(\mu) = \mathbb{E}X$ denote its mean vector, and for each coordinate $j \in \mathbb{N}$, $[\mathrm{Mean}(\mu)]_j=\mathbb{E}X_j$.
Specifically for our setting, we define an estimator $\Tilde{q}$ as a mapping from $(\{0,1\}^\NN)^n$ to $[0,1]^\NN$ where $n$ is the number of data points. The estimator will be the output of some algorithm $\mathcal{A}$.
\begin{definition}
    \label{def:learn}
    We say a collection of distributions $\QQ$ is \textbf{Uniform Mean Estimation (UME) learnable} by algorithm $\A$ if for any distribution $\mu \in \QQ,$ the algorithm $\A$ returns an estimate $\tilde q$ using $n$ i.i.d. data points $S = \{X^{(1)},X^{(2)},\ldots,X^{(n)}\}$ obtained from $\mu$ such that 
    \begin{equation*}
        \expect_{S \sim \mu ^n} \norm{\A(S)-q}_\infty = \expect_{S \sim \mu^n} \norm{\tilde q-q}_\infty \xrightarrow{n \rightarrow \infty} 0
    \end{equation*}
where $q = \text{Mean}(\mu)$. A collection of distributions $\QQ$ is UME-learnable if there exists an algorithm $\A$ such that $\QQ$ is UME-learnable by $\A$. 
\end{definition}
For a collection of distributions $\QQ$, we can define the corresponding collection of mean vectors as
\begin{equation}
    \label{eqn:meanvector}
    \text{Mean}(\QQ) = \{q \in [0,1]^\NN: q = \text{ Mean}(\mu) \text{ for some } \mu \in \QQ \}
\end{equation}
\begin{definition}
    \label{def:count}
    We say a collection of distributions $\QQ$ has a \textbf{countable }$\mathbf{\eps-}$\textbf{cover} for its mean if there is some $\QQ_\eps$ such that for any $q \in \text{Mean}(\QQ)$, there exists $q_\eps \in \QQ_\eps$ such that $\norm{q - q_\eps}_\infty < \eps$ and $\QQ_\eps$ is countable.
\end{definition}
\begin{definition}
    \label{def:sep}
    We say a collection of distributions $\QQ$ has \textbf{separable} mean vectors if for every $\eps>0$ there exists a countable $\eps-$cover for Mean$(\QQ)$. A collection of distributions $\QQ$ has \textbf{non-separable} mean vectors if for some $\eps>0$ there does not exist a countable $\eps-$cover for Mean$(\QQ)$.
\end{definition}

\begin{definition}
\label{def:ball}
    Given a collection $\mathcal{Q}$, let $\mathbf{\mathcal{B}}(q,\eps)$ denote the ball of radius $\eps$ around the vector $q$ under the $\ell_\infty$ norm defined as follows: 
    \[\mathcal{B}(q,\eps) = \curlof{q' \in \QQ: \norm{q-q'}_\infty<\eps}.\]
\end{definition}

\paragraph{Main Results} Here is a summary of our main results:
\begin{itemize}
\itemsep -3pt
    \item If $\QQ$ is countable then $\QQ$ is UME-learnable (Theorem \ref{lemma:countable}).
    \item If $\QQ$ has separable means vectors then $\QQ$ is UME-learnable (Theorem \ref{thm:separable}).
    \item UME-learnability is closed under countable unions (Theorem \ref{thm:countableunions}).
    \item A UME-learnable collection of distributions with non-separable mean (Proposition \ref{Prop:Tree}).
    \item A discussion of UME-learnability that is uniform over the collection of distributions in Appendix \ref{app:uniform-learnability}.
\end{itemize}
\section{Separability is Sufficient for UME-learnability}
\label{sec:separable}

In this section, we are interested in finding conditions on the collection of distributions $\QQ$ to guarantee UME-learnability. We focus on the collection of mean vectors, Mean$(\QQ)$, and show that their separability is a sufficient condition for UME-learnability. By Definition \ref{def:sep}, for every $\eps>0$ we obtain a countable $\eps-$cover of Mean$(\QQ)$. We use this cover to provide an $\eps-$ approximation of the true underlying distribution using Algorithm \ref{alg:epsapprox}.\\
\vspace{-6mm}
\begin{algorithm}[H]
\caption{$\eps-$approximate $(\QQ, n>0,\eps \ge 0)$}
\label{alg:epsapprox}
Initialize $\QQ_\eps = \{q^1,q^2,\ldots\}$ as the countable $\eps-$cover of Mean$(\QQ)$\\
Let $i \gets 1$ and $\hat q$ be the empirical mean estimator computed using the training data.\\
\textbf{while} {there exists $j < n$ with $\lvert q^{i}_j - \hat{q}_j \rvert > \sqrt{\frac{3\log n}{n}}+\eps$} and $ i \le n $ \textbf{do}{
    $i \gets i+1$
}\\
\Return{$q^{i}$}
\end{algorithm}
Given an $\eps-$ cover of Mean($\QQ$), Algorithm \ref{alg:epsapprox} will find the first vector in the cover that is $\eps-$close to the true mean vector. 
We leverage the fact that a vector that is not $\eps-$close to the true mean vector is $\eps$ far in at least one coordinate. We can rule out incorrect vectors by focusing on coordinates where the empirical mean closely matches the true mean. We focus on the first $n$ coordinates, as Hoeffding's inequality provides strong concentration guarantees for them. We return the first vector that is within the confidence bound provided by Hoeffding's inequality on the first $n$ coordinates.
\begin{lemma}
    \label{lemma:epsapprox}
    If collection of distributions $\QQ$ has a countable $\eps-$cover for Mean$(\QQ)$ then for any $\mu \in \QQ$ with probability $1$ there exists a data size $n_0$ such that for all $n>n_0$ the estimator $\tilde q$ returned by Algorithm \ref{alg:epsapprox} satisfies
    \[\norm{\tilde q - q}_\infty \le \eps \] where $q = $ Mean$(\mu)$.
\end{lemma}
\begin{proof}
Let $\QQ$ be a collection of distributions and $\eps>0$ be given. Let $\QQ_\eps$ be the countable $\eps-$cover of $\text{Mean}(\QQ)$ under the $\ell_\infty$ norm. Let $\mu^{*}$ be the true underlying distribution and let $q^{*} =$ Mean$(\mu^*)$. Let $q^{i^*_\eps} \in \QQ_\eps$ be the first vector such that $\norm{q^*-q^{i^*_\eps}}_\infty \le \eps$. We refer to $q^{i^*_\eps}$ as the $\eps-$approximating vector.  Let $q^{1},q^{2},\ldots,q^{i^*_\eps-1}$ be the vectors appearing before $q^{i^*_\eps}$. Hence, by definition, there exists some coordinate for which the deviation is at least $\eps$. Therefore, for $i<i^*_\eps$, we define
\begin{equation*}
\label{eqn:jiappx}
    j_i = \min \{j \in \NN: \modu{q^{i}_j - q^{*}_j}>\eps\}
\end{equation*} 
Our task is to find $q^{i^*_\eps}$. 
When we set the deviation between the empirical mean from the true mean at a particular coordinate as $\sqrt{\frac{3\log n}{n}}$ by Hoeffding inequality (\cite{Hoeffding1963}), we obtain with probability at least $1-\frac{2}{n^6},\modu{q^*_j-\hat q_j}< \sqrt{\frac{3\log n}{n}}$ for any coordinate $i$. Consequently, in Algorithm \ref{alg:epsapprox} we test the first $n$ coordinates. And as the $\eps-$approximating vector is $\eps$ far from the true vector, we allow an $\eps$ slack. Thus, we have the following test.
\begin{equation*}
\label{eqn:testappx}
    \forall j<n, \cov{|}{q^{i}_j-\hat q_j}|<\sqrt{\frac{3\log n}{n}} + \eps
\end{equation*}
For all $i<i^*_\eps$ let
\begin{equation*}
\label{eqn:gamma}
    \gamma_i =\modu{q^{i}_{j_i} - q^{*}_{j_i}}-\eps
\end{equation*} 
We note $\gamma_i>0$ by definition of the $\eps-$approximating vector and $j_i$.\\
As we wish to focus on the first $n$ coordinates, we need to ensure $n$ is large enough to include the coordinates that differentiate the vectors from the true mean vector by at least $\eps$. We also need to ensure that for any $i<i^*_\eps, q^i$ is not accidentally accepted due to the confidence bound given by Hoeffding inequality. In addition, $q^{i^*_\eps}$ should be analyzed by the algorithm. Hence, $n$ should be sufficiently large such that
\begin{equation}
    \label{eqn:ncondappx}
    n \ge i^*_\eps \quad \quad  n \ge \max_{i<i^*_\eps}j_{i} \quad \text{and } \quad 
    \min_{i<i^*_\eps} \gamma_i > 2\sqrt{\frac{3\log n}{n}}
\end{equation}
We want to ensure that the event $F_n$, that the algorithm returns any of the vectors preceding the $\eps-$ approximating vector, and the event $G_n$ that the algorithm does not return the $\eps-$ approximating vector after obtaining $n$ data points, do not occur infinitely often.\\
We start by analyzing the probability of $F_n$. We apply the union bound together with the second constraint in equation \eqref{eqn:ncondappx}.
\begin{equation}
    \begin{aligned}
    \label{eqn:cond1step1appx}
        \prob\parof{\exists i<i^*_\eps : \forall j<n, \modu{q^{i}_j - \hat q_j}<\sqrt{\frac{3\log n}{n}}+\eps}  
        & \le \sum_{i=1}^{i^*_\eps-1}\prob\parof{\modu{q^{i}_{j_{i}} - \hat q_{j_{i}}}<\sqrt{\frac{3\log n}{n}}+\eps} 
    \end{aligned}
\end{equation}
We use the third constraint and triangle inequality to obtain 
\begin{equation}
    \begin{aligned}
        \label{eqn:cond1step2appx}
        \modu{q^{i}_{j_{i}} - \hat q_{j_{i}}} \ge \modu{q^{i}_{j_{i}} - q^{*}_{j_{i}}}- \modu{\hat q_{j_{i}}-q^{*}_{j_{i}}} =\gamma_i+\eps-\modu{\hat q_{j_{i}}-q^{*}_{j_{i}}} \ge 2\sqrt{\frac{3\log n}{n}}+\eps-\modu{\hat q_{j_{i}}-q^{*}_{j_{i}}}  
    \end{aligned}
\end{equation}
We combine equations \eqref{eqn:cond1step1appx},\eqref{eqn:cond1step2appx}, use the first constraint and apply Hoeffding inequality to obtain 
\[\sum_{i=1}^{i^*_\eps-1}\prob\parof{\modu{q^{i}_{j_{i}} - \hat q_{j_{i}}}<\sqrt{\frac{3\log n}{n}}+\eps}  \le \sum_{i=1}^{i^*_\eps-1} \prob\parof{\modu{\hat q_{j_i}-q^*_{j_i}}>\sqrt{\frac{3\log n}{n}}}\le \frac{2(i^*_\eps-1)}{n^6} \le \frac{2}{n^5}\]
Similarly, we can analyze the probability of the event $G_n$ using the union bound and Hoeffding inequality to obtain
\begin{equation*}
    \begin{aligned}
        \label{eqn:cond2appx}
    \prob\parof{\exists j<n: \modu{q^{i^*_\eps}_j - \hat q_j} \ge \sqrt{\frac{3\log n}{n}}} & \le \sum_{j=1}^n \prob\parof{\modu{q^{i^*_\eps}_j - \hat q_j}\ge \sqrt{\frac{3\log n}{n}}} \le \sum_{j=1}^n \frac{2}{n^6} \le \frac{2}{n^5}
    \end{aligned}
\end{equation*}
We define the event $E_n$ as the occurrence of either $F_n$ or $G_n$. By our previous analysis we obtain $\prob(E_n) \le \frac{4}{n^5} $. We note that $\sum_{n=1}^\infty \prob(E_n) \le \sum_{n=1}^\infty \frac{4}{n^5} <\infty.$ Hence, we can use the First Borel-Cantelli Lemma to conclude that with probability $1$ there exists $n_0>0$ such that for all $n>n_0$ the algorithm successfully finds $q^{i^*}_\eps$. 
\end{proof}
\vspace{-3mm}
As a direct by-product, we can show that any countable collection of distributions is UME-learnable. 
\begin{theorem}
\label{lemma:countable}
    If $\QQ$ is countable then $\QQ$ is UME-learnable by Algorithm \ref{alg:epsapprox} with $\eps=0$.
\end{theorem}
\begin{proof}
Let $\QQ$ be a countable collection of distributions. We note that $\QQ$ is a $0-$cover of itself. We use Lemma \ref{lemma:epsapprox} with $\eps =0$ to obtain
with probability $1$, for any $\mu \in \QQ$ with $q = $ Mean$(\mu)$, there exists $n_0$ such that for all $n>n_0$ the estimate $\tilde q$ returned by Algorithm \ref{alg:epsapprox} satisfies $\norm{\tilde q -q} = 0$. As $\norm{\tilde q-q}_\infty \le 1$ by the Dominated Convergence Theorem we obtain $\expect\norm{\tilde q-q}_\infty \xrightarrow{n \rightarrow \infty} 0$.
\end{proof} 
We can now consider a collection of distributions that have separable mean vectors and show that they are UME-learnable by Algorithm \ref{alg:sep}. \\ \\
\vspace{-7mm}
 \begin{algorithm}
 \begin{algorithmic}
\caption{Separable $(\mathcal{Q}, n>0)$}
\label{alg:sep}
\State Initialize $\PP \gets \text{Mean}(\QQ)$ where $\text{Mean}(\QQ)$ is as in equation \eqref{eqn:meanvector} 
\State Let $\tilde q \gets \emptyset, k \gets 1 $ 
\State \textbf{while} {$\PP$ is not empty and $k \le \log n$} \textbf{do}{
    \State \hspace{\algorithmicindent} $\eps_k \gets \frac{1}{2^k}, \tilde q \gets $ any $q \in \PP$ 
    \State \hspace{\algorithmicindent} Run Algorithm \ref{alg:epsapprox}$(\QQ,n,\eps_k)$ to obtain $q^{k}$
    \State \hspace{\algorithmicindent} $\PP \gets \PP \cap \mathcal{B}(q^{k},\eps_k)$
    \State \hspace{\algorithmicindent} $k \gets k+1$
}
\State \Return{$\tilde q$}\;
\end{algorithmic}
\end{algorithm}
\\ For a collection of distributions that have separable mean vectors, we run Algorithm \ref{alg:epsapprox} at countably many resolutions $\eps_k = 2^{-k}$ and take a vector that lies in the intersection of $\eps_k-$balls around the vectors returned by Algorithm \ref{alg:epsapprox}. Let $K$ be the value such that $\norm{q^k-q^*}_\infty\le \eps_k$ for every $k \le K$ where Algorithm \ref{alg:epsapprox} returns $q^k$ for $\eps_k$ resolution. Hence, $q^*$(the true mean vector) is in the intersection of these balls. The algorithm selects a vector in the last non-empty intersection, thus yielding a $2\eps_K$ approximation of the true mean vector. Increasing $n$ yields finer approximations, ensuring asymptotic convergence.
\begin{theorem}
\label{thm:separable}
    If the collection of distributions $\QQ$ has separable mean vectors, then $\QQ$ is UME-learnable by Algorithm \ref{alg:sep}.
\end{theorem}
\vspace{-2mm}
\begin{proof}
We prove the theorem by presenting an algorithm that returns an estimate arbitrarily close to the true underlying mean. The analysis relies on obtaining a sufficiently large training set. As we increase the size of the training set, we obtain increasingly accurate approximations of the true mean vector. From Lemma \ref{lemma:epsapprox}, we know that if a countable $\eps-$cover exists, then we can find an $\eps-$ approximation of the true mean vector. Here, we exploit Algorithm \ref{alg:epsapprox} to establish UME-learnability for a separable collection of distributions.\\
Let $\mu^* \in \QQ$ be the true distribution, and let $q^* = $ Mean$(\mu^*)$. Let $n$ denote the number of data points obtained. We define $\eps_k = 2^{-k}$. We denote the countable $\eps_k-$cover of Mean$(\QQ)$ by $\QQ_k$, and let $q^k$ be the estimate returned by Algorithm \ref{alg:epsapprox} for $\eps = \eps_k$.By Lemma \ref{lemma:epsapprox}, with probability $1$ there exists $n_k$ such that for every $n>n_k,$ the estimator $q^k$ satisfies $\norm{q^k-q^*}_\infty \le \eps_k.$ \\
When $n>n_k,$ we say Algorithm \ref{alg:sep} has converged for $\eps_k$. Let $K$ be the largest value such that for all $k \le K$ the algorithm has converged for $\eps_k$.
Let $q \in $ Mean$(\QQ) \cap\bigcap_{k \le K}\mathcal{B}(q^k, \eps_k)$ be any vector in the intersection of the balls for the converged values for $\eps_k$. Also note that, since this algorithm has converged for all $k \le K$, the true mean vector lies in the intersection, making it non-empty. Furthermore, using the triangle inequality, we obtain 
\begin{equation}
    \label{eqn:convergevstrue}
    \norm{q-q^*}_\infty \le \norm{q-q^K}_\infty + \norm{q^K-q^*}_\infty \le 2 \eps_K
\end{equation}
Note that the algorithm does not necessarily stop at $k=K$; rather, it continues until the intersection of the balls around the vectors returned by Algorithm \ref{alg:epsapprox} becomes empty. Let $\mathscr{K}$ be the largest $k$ such that the intersection of the balls is non-empty. The algorithm then returns $\tilde q \in $ Mean$(\QQ) \cap \bigcap_{k \le \mathscr{K}} \mathcal{B}(q^k,\eps_k)$. Since the intersection for the first $K$ balls is non-empty, it follows that $\mathscr{K} \geq K$. In particular, it further implies that $\tilde q$ is in Mean$(\QQ)\cap\bigcap_{k \le K} \mathcal{B}(q^k,\eps_k)$. Thus, using equation \eqref{eqn:convergevstrue} we conclude that
\begin{equation*}
    \label{eqn:intersectvsconverge}
    \norm{\tilde q-q^*}_\infty < 2 \eps_K
\end{equation*}
Lemma \ref{lemma:epsapprox} holds simultaneously for all $k \in \NN$ by union bound. Hence, with probability $1$ we obtain,
\[\norm{\tilde q - q^*}_\infty \le 2 \min \curlof{\eps_k:n>n_k}\]
and since $n_k < \infty$ for every $k \in \NN, \lim_{n \rightarrow \infty} \min\curlof{\eps_k:n>n_k}=0$.\\
And as $\norm{\tilde q-q}_\infty \le 1$ by the Dominated Convergence Theorem we obtain $\expect\norm{\tilde q-q}_\infty \xrightarrow{n \rightarrow \infty} 0$.
\end{proof}
\vspace{-5mm}
\section{Examples}
Section \ref{sec:separable} shows us the sufficiency of separability in the mean as a characterization for UME-learnability. We consider proposition $1$ from \cite{cohen2023localglivenkocantelli}, which shows that the following collection of distributions is UME-learnable. We show it is also separable and therefore UME-learnable.
\begin{equation*}
\label{eqn:mleminimaxprop}
\QQ_{prop} = \curlof{\mu: \mu = \text{Prod}(q) \text{ such that for all }  j \in \NN, \modu{q_j-\frac{1}{2}}\le \frac{c}{\sqrt{j}}}
\end{equation*}for a universal constant $c>0$. 
\begin{proposition}
    $\QQ_{prop}$ has separable mean vectors.
\end{proposition}
\begin{proof}
We wish to show that for every $\eps>0$ we can provide a countable $\eps-$cover.\\
Let $\eps>0$ be given. Let $j_\eps =\cov{\lceil}{\frac{c^2}{\eps^2}}{\rceil}$ we define the $\eps-$covering set $\QQ_\eps$ as follows:
\[\QQ_\eps = \curlof{p \in [0,1]^\NN:p_j \in \mathbb{Q} \text{ if } j\le j_\eps \text{ and } p_j = \frac{1}{2} \text{ otherwise}}\]
Note that for any vector $q \in \text{Mean}(\QQ)$ and any vector $p \in \QQ_\eps$,
\[||p- q||_\infty = \max\curlof{ \max_{i \le j_\eps}| p_i - q_i |, \sup_{j > j_\eps}|p_i-q_i|} \le \max\curlof{ \max_{i \le j_\eps}| p_i - q_i |, \eps} \] 
As rationals can arbitrarily approximate any real, there exists $p \in \QQ_\eps$ such that $\max_{j \le j_\eps}|p_i-q_i|<\eps$. Hence $\QQ_\eps$ is an $\eps-$cover of $\QQ$, and as it consists of rational numbers for finite coordinates, it is countable. 
\end{proof}
The sufficiency of separability for UME-learnability leads to a natural question. 
\begin{center}\textit{Is separability necessary for UME-learning?}
\end{center}
We answer this in the negative. Consider the following collection of distributions
\begin{equation*}
    \label{eqn:binaryvectors}
    \QQ_{bin} = \curlof{\mu: \text{Mean}(\mu) \in \{0,1\}^\NN}
\end{equation*}
The collection of distributions whose means are the set of all binary vectors is trivially UME-learnable. With one data point, we know the exact underlying distribution used for sampling, as the realization will be $0$ only if the mean value was $0$, and it will be $1$ only if the mean value was $1$.\\
We also note that $\QQ_{bin}$ has non-separable mean vectors. If possible, $\QQ_{bin}$ have separable mean vectors. Let $\bar \QQ$ be a countable $\frac{1}{2}-$cover of Mean$(\QQ_{bin})$. We note that for any $q,q' \in \text{Mean}(\QQ),$ if $q \in \mathcal{B}(p,\frac{1}{2})$ for some $p \in \bar \QQ$ then $q' \notin \mathcal{B}(p,\frac12)$ as $\norm{q-q'}_\infty=1$. Hence, there is only one element of $\QQ$ in every ball of radius $\frac{1}{2}$ around any $p \in \bar \QQ$, making $\bar \QQ$ uncountable and thus contradicting our assumption.\\
Although the above example provides a trivial counterexample to the necessity of separability in the mean for UME-learnability, there are non-separable collections of distributions that are not UME-learnable. For instance, consider the following collection of distributions.
\begin{equation*}
\label{eqn:qtert}
\QQ_{tert} = \curlof{\mu: \mu = \text{Prod}(q) \text{ where } q \in \curlof{\frac13,\frac23}^\NN}
\end{equation*}
We note that $\QQ_{tert}$ is also non-separable as for any $q,q' \in \QQ_{tert}, \norm{q-q'}_\infty = \frac{1}{3}$. Hence, we can apply an argument similar to the one used to show the non-separability of $\QQ_{bin}$. We also refer to the proof of Theorem $1$ in \cite{cohen2025empiricalmeanminimaxoptimal}, which implicitly proves that $\QQ_{tert}$ is not UME-learnable. \\ 
The collection of distributions whose means are binary vectors is a trivial example of a non-separable but UME-learnable collection of distributions. We present a non-trivial example of non-separable classes that is UME-learnable using techniques fundamentally different from those used in the separability-based analysis or methods used in the literature. \\  
We define a collection of distributions $\QQ_{tree}$ using their respective mean vectors. We consider a binary tree and label it using a mean vector by traversing it in a level order fashion. Formally, for every mean vector $q = (q_1,q_2,\ldots)$, the root corresponds to $q_1$, the left child of $q_1$ corresponds to $q_2$, the right child of $q_1$ corresponds to $q_3$, the left child of $q_2$ corresponds to $q_4$, the right child of $q_2$ corresponds to $q_5$ and so on. Because the mean vector is infinite, the binary tree has infinite depth. Finally, for any branch, i.e., a root-to-leaf path in the tree, we assign all their corresponding coordinates with the value $\frac{2}{3}$, whereas all other coordinates are given a value of $\frac{1}{3}$.\\
Hence, we can define $\QQ_{tree}$ as a collection of distributions for all such mean vectors as follows
\begin{equation*}
    \label{eqn:tree}
    \QQ_{tree} = \curlof{\mu : \mu = \text{Prod}(q) \text{ where } q \text{ satisfies the structure given above}}
\end{equation*}
We note that $\QQ_{tree}$ has non-separable mean vectors as for any $q, q' \in \QQ_{tree}, \norm{q-q'}_\infty = \frac{1}{3}.$ Hence, we can use an argument similar to the one used to show the non-separability of Mean$(\QQ_{bin})$.
\begin{proposition}
\label{Prop:Tree}
    $\QQ_{tree}$ is UME-learnable.
\end{proposition}
\paragraph{Proof Sketch} To learn $\QQ_{tree}$, we leverage the tree structure embedded in the collection of mean vectors. We note that finding the true underlying mean vectors is equivalent to identifying the branch of the tree labeled $\frac23$. We calculate the limiting average of the empirical means along every branch of the tree and return the branch for which it is exactly $\frac{2}{3}$. The algorithm works as we can show that for all branches other than the true branch, the limiting average is not $\frac{2}{3}$ uniformly. The formal proof, along with the motivating idea, is provided in Appendix \ref{appendix:tree}. \hfill $\blacksquare$.\\
In many learning-theory problems, the notion of a bad structure arises, and if such a structure appears, the learning problem is considered hard or not learnable. We can regard $\QQ_{tert}$ as a bad structure for UME-learnability, but we can show that it is a substructure of another problem that is UME-learnable, as seen in the following example.\\
    Consider $\QQ_{round}$ defined as follows:\\
    \[
    \QQ_{round}  = \curlof{\mu: \mu = \text{Prod}(q) \text{ such that } q_{2n-1} \in \curlof{\frac13,\frac23} \text{ and } q_{2n} = \1\sqof{q_{2n-1}=\frac23} \text{ for } n  \in \NN}
    \]
\begin{proposition}
\label{prop:round}
    $\QQ_{round}$ is UME-learnable.
\end{proposition}
\begin{proof}
    Let $\mu^* \in \QQ_{round}$ be the underlying distribution. Let $q^*$ = Mean$(\mu^*)$. Let $X \sim \mu^*$ be a data point. As $q_{2n} \in \{0,1\}$ we can find them using the value of $X_{2n}$ for every $n \in \NN$. And as $q_{2n} = \1\sqof{q_{2n-1}=\frac{2}{3}}$, $q_{2n-1}$ can be inferred using $q_{2n}$. 
\end{proof}
\vspace{-8mm}
\section{{UME-}learnability is closed under countable unions.}
\label{sec:countableunion}
The conjecture from \cite{cohen2025empiricalmeanminimaxoptimal} looks at countable collections of distributions $\QQ$ with some specific properties to ascertain the UME-learnability of  LGC$ \cup \QQ$. We claim a collection of distributions $\QQ = \cup_{i \in \NN} \QQ_i$ where $\QQ_i$ is UME-learnable by algorithm $\A_i$ which returns an estimate $\tilde q^i$ is also UME-learnable using the following algorithms,
\begin{algorithm}
    \caption{Survival Test $(i,\eps,n,\parof{\tilde q^1,\tilde q^2,\ldots, \tilde q^n},\hat q)$}
    \label{alg:survival}
    \begin{algorithmic}
    \State Initialize  wins $\gets 0$ 
    \State \textbf{for} {$t$ goes from $1 $ to $n$} \textbf{ do }{
    \State \hspace{\algorithmicindent} \textbf{if} for every $j \in \NN \modu{\tilde q^t_j - \tilde q^i_j} \le 4 \eps$ \textbf{then} wins $\gets $ wins $+1$
    \State \hspace{\algorithmicindent} \textbf{else} 
    \State  \hspace{\algorithmicindent}\hspace{\algorithmicindent}{ $J = \min\{j \in \NN: \modu{\tilde q^t_j - \tilde q^i_j} > 4 \eps\}$
    \State \hspace{\algorithmicindent}\hspace{\algorithmicindent}\textbf{if} $\modu{\hat q_J - \tilde q^i_J} < \eps+ \sqrt{\frac{3\log n}{n}}$ \textbf{then}  wins $\gets $ wins $+ 1$
    }
    }
    \State \textbf{if} wins is equal to $n$ \textbf{then} return ``pass" \textbf{otherwise} return ``fail'' 
    \end{algorithmic}
\end{algorithm}
\begin{algorithm}
    \caption{Countable union $(\QQ,2n>0,(\A_1,\A_2,\ldots))$}
    \label{alg:countableunion}
    \begin{algorithmic}
    \State We split the $2n$ training data into a training set $S_1$ and a validation set $S_2$ each of size $n$.
    \State $\mathcal{P} \gets $Mean$(\QQ), \tilde q \gets \emptyset,k \gets 1$ 
    \State We consider the first $n$ algorithms and run $\A_1,\A_2, \ldots,\A_n$ on $S_1$ to obtain $\tilde q^1,\tilde q^2,\ldots, \tilde q^n$ resp.
    \State We compute the empirical mean estimator using $S_2$ to obtain $\hat q$
    \State \textbf{while} {$\PP$ is not empty} \textbf{do}{
    \State \hspace{\algorithmicindent}$\eps \gets \frac{1}{2^k},\tilde q \gets \text{any } q \in \PP$ 
    \State \hspace{\algorithmicindent} \textbf{for} {$i$ goes from $1$ to $n$} \textbf{do}{
    \State \hspace{\algorithmicindent} \textbf{if} {Algorithm \ref{alg:survival} $(i,\eps,n,(\tilde q^1,\ldots, \tilde q^n),\hat q)$ returns ``pass''} \textbf{then}{
    \State \hspace{\algorithmicindent}\hspace{\algorithmicindent}$\PP \gets \PP\; \cap \; \mathcal{B}(\tilde q^i,5\eps)$ 
    
    \State \hspace{\algorithmicindent}\hspace{\algorithmicindent}$k \gets k+1$
    }
    }
    }
    \State \Return $\tilde q$
    \end{algorithmic}
\end{algorithm}

An estimator that survives Algorithm \ref{alg:survival} will be a $5\eps-$approximation of the true underlying mean vector for any $\eps>0$ with high probability for a sufficiently large amount of training data. This algorithm is used as a subroutine for the algorithm \ref{alg:countableunion}. Similar to Algorithm \ref{alg:epsapprox}, Algorithm \ref{alg:countableunion} focuses on the first $n$ algorithms and, like Algorithm \ref{alg:sep}, it chains $\eps_k-$approximations of the true underlying distribution to guarantee UME-learnability of countable unions. 
\begin{theorem}
\label{thm:countableunions}
    UME-learnability is closed under countable unions.
\end{theorem}
\begin{proof}
    Let $\QQ = \cup_{i \in \NN} \QQ_i$ be our collection of distributions where $\QQ_i$ is UME-learnable by algorithm $\A_i$ using the estimator $\tilde q^i$. Let $\mu^* \in \QQ_{i^*}$ be our true underlying distribution. Let $q^* = $ Mean$(\mu^*)$. We have been provided with $2n$ data points. We split the data into a training set $(S_1)$ and a validation set $(S_2)$, each of size $n$. Let $\A_1,\A_2,\ldots,\A_n$ denote the first $n$ algorithms and let $\tilde q^1,\tilde q^2,\ldots, \tilde q^n$ denote the estimates returns by running the respective algorithm on $S_1$. Let $\hat q$ denote the empirical mean estimator calculated using $S_2$. Algorithm \ref{alg:survival} checks whether a particular estimator could approximate the true underlying mean vector. We say an estimator $\tilde q^1$ wins against another estimator $\tilde q^2$ if $\norm{\tilde q^1-\tilde q^2}_\infty\le4\eps$ or $\modu{\tilde q^1_{J}-\hat q_{J}}<\eps+\sqrt{\frac{3\log n}{n}}$ where $\modu{\tilde q^1_J-\tilde q^2_J}>4\eps$. \\
    Let $n>i^*$ and $\parof{\eps^*}_n = \sqrt{\expect_{S_1} \norm{\tilde q^{i^*}-q^*}_\infty}$ . For any $\eps>\max \curlof{\parof{\eps^*}_n, \sqrt[4]{\frac{9}{n}}}$, let $I_1 = \{i \in [n]:$$\norm{\tilde q^{i^*}-\tilde q^i}_\infty \le 4\eps\}$ and $I_2 = [n]\setminus I_1$. 
    For any $i \in I_2$ let $j_i = \min\curlof{j \in \NN:\modu{\tilde q^i_j - \tilde q^{i^*}_j}>4\eps }$.\\
    As $\QQ_{i^*}$ is UME-learnable by $\A_{i^*},$ we focus on $\tilde q^{i^*}$. $\tilde q^{i^*}$ will win against every $\tilde q^i$ for $i \in I_1$ by definition of $I_1$. Therefore, we focus on the event $E_n$ of $\tilde q^{i^*}$ not winning against $\tilde q^i$ for some $i \in I_2$. We analyze the probability of $E_n$ by conditioning on the event $\norm{\tilde q^{i^*}-q^*}_\infty<\eps$. We apply the union bound to focus on comparing $\tilde q^{i^*}$ with $\tilde q^i$ for $i \in \NN$ and the triangle inequality to focus on the deviation from the true underlying mean. Next, we use the condition. We finally apply Markov inequality (\cite{markov1884}) and Hoeffding inequality (\cite{Hoeffding1963}), as detailed below.
    \begin{equation*}
        \label{eqn:i^*pass}
        \prob(E_n) \le \prob\parof{\exists i \in I_2: \modu{\hat q_{j_i}-\tilde q^{i^*}_{j_i}}\ge \eps+\sqrt{\frac{3 \log n}{n}} \Bigg | \norm{\tilde q^{i^*}-q^*}_\infty<\eps} +\prob\parof{\norm{\tilde q^{i^*}-q^*}_\infty\ge \eps}
        \end{equation*}
        \begin{equation*}
         \le \sum_{i=1}^n \prob\parof{\modu{\hat q_{j_i}-q^*_{j_i}}+\modu{q^*_{j_i}-\tilde q^{i^*}_{j_i}}\ge \eps+\sqrt{\frac{3 \log n}{n}} \Bigg | \norm{\tilde q^{i^*}-q^*}_\infty<\eps} + \prob\parof{\norm{\tilde q^{i^*}-q^*}_\infty\ge \eps}
        \end{equation*}
        \begin{equation}
        \le n \cdot 2\exp\parof{-2n\parof{\sqrt{\frac{3\log n}{n}}}^2}+\frac{\expect\norm{\tilde q^{i^*}-q^*}_\infty}{\eps}\le \frac{2}{n^5}+\parof{\eps^*}_n
    \end{equation}
    We also note that if Algorithm \ref{alg:survival} ``passes'' $\tilde q^{i^*}$ then none of $i \in I_2$ can pass the test because for any $i \in I_2$, $\modu{\tilde q^{i}_{j_i}-\hat q_{j_i}}\ge \modu{\tilde q^i_{j_i}-\tilde q^{i^*}_{j_i}}-\modu{\tilde q^{i^*}_{j_i}-\hat q_{j_i}} \ge 2\eps$ since $n>\frac{9}{\eps^4}$. $\tilde q^i$ for some $i \in I_1$ might survive algorithm \ref{alg:survival}. But since $\norm{\tilde q^i-\tilde q^{i^*}}_\infty\le 4\eps$, with probability $1-\parof{\eps^*}_n-\frac{2}{n^5}$ any estimator that survives algorithm \ref{alg:survival} is a $5\eps-$approximation of the true mean vector.\\
    Algorithm \ref{alg:countableunion} exploits all the $5\eps_k-$approximating vectors obtained on running Algorithm \ref{alg:survival} for the first $n$ algorithms for $\eps_k = 2^{-k}, k \in \NN$ and obtains a vector which is a $10\eps_k-$approximating vector for all $k \in \NN$ simultaneously asymptotically as we argue as follows.\\
    Let $K = \min\curlof{\cov{\lfloor} {\frac{1}{\sqrt{\parof{\eps^*}_n+\frac{2}{n^5}}}}{\rfloor},\cov{\lfloor}{\log \parof{\frac{1}{(\eps^*)_n}}}{\rfloor},\cov{\lfloor}{\frac{1}{4}\log \frac{n}{9}}{\rfloor}}$ and let $\mathcal{R}=$Mean$(\QQ) \cap  \bigcap_{k \le K}\mathcal{B}(q^k,5\eps_k)$ where $q^k$ is the estimator that ``passes'' Algorithm \ref{alg:survival} for $\eps=\eps_k$.
    By union bound we obtain with  probability $1-K\parof{\parof{\eps^*}_n+\frac{2}{n^5}}\ge 1-\sqrt{\parof{\eps^*}_n+\frac{2}{n^5}}$ for every $k \in [K], \norm{ q^k-q^*}<5\eps_k$. \\
    Hence with probability $1-\sqrt{\parof{\eps^*}_n+\frac{2}{n^5}}, \mathcal{R}$ is non-empty as the true mean vector will be in $\mathcal{R}$. \\
    The algorithm does not halt after the first $K$ rounds; rather, it continues until the intersection of the balls around the vectors returned by Algorithm \ref{alg:survival} becomes empty. Let $\mathscr{K}$ be the largest $k$ such that the intersection is non-empty. Let $\mathcal{T} = $ Mean$(\QQ) \cap  \bigcap_{k \le \mathscr{K}}\mathcal{B}(q^k,5\eps_k)$. Due to our previous argument, $\mathcal{T} \subset \mathcal{R}$. Hence, the estimate $\tilde q$ returned by the algorithm is in $\mathcal{R}$. Therefore, we can conclude with probability $1-\sqrt{\parof{\eps^*}_n+\frac{2}{n^5}}$
    \begin{equation}
    \label{eqn:boundonexpect}
        \norm{\tilde q-q^*}_\infty\le \norm{\tilde q-q^K}_\infty+ \norm{ q^K-q^*}_\infty \le 10\eps_K.
    \end{equation}
    We note that by definition of UME-learnability (Definition \ref{def:learn}), $(\eps^*)_n \xrightarrow{n \rightarrow \infty}0$ and $K \xrightarrow{n \rightarrow \infty} \infty$.
    Therefore, by Equation \eqref{eqn:boundonexpect}, $\expect \norm{\tilde q-q^*}_\infty \le 10 \eps_K+\sqrt{\parof{\eps^*}_n+\frac{2}{n^5}} \xrightarrow{n \rightarrow \infty}0.$
    \end{proof}
\vspace{-5mm}
\section{Conclusion}
In this paper, we discuss uniform convergence beyond the paradigm of $ P-$Glivenko-Cantelli by studying more general types of estimators than the empirical mean estimator.
We introduced UME-learnability to characterize when collections of distributions on $\{0,1\}^{\mathbb N}$ admit uniform mean estimation by arbitrary estimators. We showed that if a collection of distributions $\mathcal{Q}$ has separable mean vectors, then $\mathcal{Q}$ is UME-learnable. We further demonstrated that separability is not necessary by constructing a non-separable, tree-structured collection that is nevertheless UME-learnable via techniques distinct from the separability-based analysis. Finally, we proved that UME-learnability is closed under countable unions, thereby resolving the conjecture of \cite{cohen2025empiricalmeanminimaxoptimal} and extending it beyond the two-collection setting considered there.\\
Uniform convergence is often used in the design and analysis of algorithms for problems such as classification. One natural application of our more general estimators would be as an alternative to empirical risk minimization in those learnable problems by minimizing the estimated mean losses beyond empirical means. 
\section{Extensions and open problems}
This work opens several natural directions for further investigation. Some partial progress on these questions is already included in the appendix, while others remain open and appear to require new ideas.
\begin{itemize}[leftmargin=*]
    \item Throughout this work, we focus on distributions indexed by a countable coordinate set. A natural extension we have studied in Appendix \ref{appendix:uncountable} is to allow an uncountable coordinate set. In Theorem~\ref{thm:modseparable}, we show that separability of the mean space remains a sufficient condition for UME-learnability even in this more general setting.
    \item While we show that separability of the mean space implies UME-learnability, Proposition~\ref{Prop:Tree} demonstrates that this condition is not necessary. This raises the problem of identifying necessary and sufficient conditions for UME-learnability when we do not restrict the mean vectors of a collection of distributions to be separable.
    An especially challenging open question is to characterize UME-learnability in the non-separable regime when the coordinate set is uncountable.
    \item Another extension, discussed in Appendix \ref{app:uniform-learnability}, is regarding uniform convergence over the function class as well as the underlying collection of distributions. When the mean vectors of a collection of distributions are totally bounded, we can provide an upper bound on the expected estimation error. An open problem is to provide a complete characterization of optimal \emph{uniform} and \emph{universal} rates of UME-learnability.
\end{itemize}
\newpage
\bibliography{alt2026-style/bibliography}

\begin{thebibliography}{11}
\providecommand{\natexlab}[1]{#1}
\providecommand{\url}[1]{\texttt{#1}}
\expandafter\ifx\csname urlstyle\endcsname\relax
  \providecommand{\doi}[1]{doi: #1}\else
  \providecommand{\doi}{doi: \begingroup \urlstyle{rm}\Url}\fi

\bibitem[Blanchard et~al.(2024)Blanchard, Cohen, and Kontorovich]{blanchard2024correlatedbinomialprocess}
Moïse Blanchard, Doron Cohen, and Aryeh Kontorovich.
\newblock Correlated {B}inomial {P}rocess, 2024.
\newblock URL \url{https://arxiv.org/abs/2402.07058}.

\bibitem[Cohen and Kontorovich(2023)]{cohen2023localglivenkocantelli}
Doron Cohen and Aryeh Kontorovich.
\newblock Local {G}livenko-{C}antelli, 2023.
\newblock URL \url{https://arxiv.org/abs/2209.04054}.

\bibitem[Cohen et~al.(2025)Cohen, Kontorovich, and Weiss]{cohen2025empiricalmeanminimaxoptimal}
Doron Cohen, Aryeh Kontorovich, and Roi Weiss.
\newblock The {E}mpirical {M}ean is {M}inimax {O}ptimal for {L}ocal {G}livenko-{C}antelli, 2025.
\newblock URL \url{https://arxiv.org/abs/2410.02835}.

\bibitem[Hoeffding(1963)]{Hoeffding1963}
Wassily Hoeffding.
\newblock Probability inequalities for sums of bounded random variables.
\newblock \emph{Journal of the American Statistical Association}, 58\penalty0 (301):\penalty0 13–30, March 1963.
\newblock ISSN 1537-274X.
\newblock \doi{10.1080/01621459.1963.10500830}.
\newblock URL \url{http://dx.doi.org/10.1080/01621459.1963.10500830}.

\bibitem[Markov(1884)]{markov1884}
Andrey~A. Markov.
\newblock On some applications of algebraic calculus to probabilities.
\newblock \emph{Proceedings of the Kazan Physical-Mathematical Society}, 2:\penalty0 3--20, 1884.
\newblock Originally published in Russian as ``O nekotorykh prilozheniyakh algebraicheskogo ischisleniya k veroyatnostyam''.

\bibitem[Rudin(1991)]{rudin1991functional}
Walter Rudin.
\newblock \emph{Functional Analysis}.
\newblock McGraw-Hill, New York, 2 edition, 1991.
\newblock ISBN 978-0-07-054236-5.

\bibitem[Sriperumbudur et~al.(2013)Sriperumbudur, Fukumizu, Gretton, Hyv\"{a}rinen, and Kumar]{densityEstimationInfDimExpFam}
Bharath Sriperumbudur, Kenji Fukumizu, Arthur Gretton, Aapo Hyv\"{a}rinen, and Revant Kumar.
\newblock Density estimation in infinite dimensional exponential families, 2013.
\newblock URL \url{https://arxiv.org/abs/1312.3516}.

\bibitem[van~der Vaart and Wellner(2023)]{vanderVaart2023}
A.~W. van~der Vaart and Jon~A. Wellner.
\newblock \emph{Weak Convergence and Empirical Processes: With Applications to Statistics}.
\newblock Springer International Publishing, 2023.
\newblock ISBN 9783031290404.
\newblock \doi{10.1007/978-3-031-29040-4}.
\newblock URL \url{http://dx.doi.org/10.1007/978-3-031-29040-4}.

\bibitem[Vapnik and Chervonenkis(1971)]{vc1971}
V.~N. Vapnik and A.~Ya. Chervonenkis.
\newblock On the uniform convergence of relative frequencies of events to their probabilities.
\newblock \emph{Theory of Probability \& Its Applications}, 16\penalty0 (2):\penalty0 264--280, 1971.
\newblock \doi{10.1137/1116025}.
\newblock URL \url{https://doi.org/10.1137/1116025}.

\bibitem[Vapnik(2006)]{Vapnik2006}
Vladimir Vapnik.
\newblock \emph{Estimation of {D}ependences {B}ased on {E}mpirical {D}ata}.
\newblock Springer New York, 2006.
\newblock ISBN 9780387342399.
\newblock \doi{10.1007/0-387-34239-7}.
\newblock URL \url{http://dx.doi.org/10.1007/0-387-34239-7}.

\bibitem[Vapnik and Chervonenkis(1981)]{vapnik1981uniform}
Vladimir~N. Vapnik and Alexey~Ya. Chervonenkis.
\newblock Necessary and sufficient conditions for the uniform convergence of means to their expectations.
\newblock \emph{Theory of Probability and Its Applications}, 26\penalty0 (4):\penalty0 532--553, 1981.

\end{thebibliography}
\newpage
\appendix
\section{An interesting UME-learnable collection of distributions that has non-separable mean vectors }
\label{appendix:tree}
The UME-learnability of the collection of distributions with the binary vectors as their corresponding mean vectors provides a trivial counterexample towards separability in the mean being a necessary condition for UME-learnability of a collection of distributions. The non UME-learnability of the collection of product measures with their mean in $\curlof{\frac13,\frac23}^\NN$ as proven implicitly in Theorem $1$ of \cite{cohen2025empiricalmeanminimaxoptimal} warrants further investigation in the setting of a collection of distributions that have non-separable mean vectors. \\
We show that the collection of distributions with non-separable mean vectors possesses an inherent structure in which the mean vectors can be infinitely sequentially fat-shattered. More formally, consider a complete binary tree of depth $d$. The nodes of the tree are labeled by integers that correspond to the coordinates of the collection of mean vectors. Each node is associated with a value $r_i \in (0,1)$. At a node with label $i$, the left edge indicates the value of the mean vector at the $i^{\mathrm{th}}$ coordinate is less than or equal to $r_i-\gamma$, whereas the right edge indicates the value is greater than or equal to $r_i+\gamma$ for some $\gamma >0$. We say a tree of depth $d$ is shattered by the mean vectors if for every branch in the tree, there exists a mean vector that follows the path as set by the nodes in the branch (For example, refer to Fig. \textcolor{blue}{1}). We say the mean vectors are infinitely shattered if for every $d \in \NN$ there exists a tree of depth $d$ that is shattered by the mean vectors.
\begin{theorem}
\label{thm:infseqfatshattered}
If a collection of distributions $\QQ$ has non-separable mean vectors, then there exists $\gamma > 0$ such that the mean vectors Mean$(\QQ)$ can be infinitely sequentially fat-shattered.
\end{theorem}
To prove Theorem \ref{thm:infseqfatshattered}, we first develop the necessary machinery. As $\QQ$ has a non-separable mean vectors, there exists $\gamma>0$ such that $\QQ$ does not have a countable $3\gamma-$cover for its means. We will use this $\gamma$ to show that Mean$(\QQ)$ is infinitely sequentially fat-shattered. We will also use lemma \ref{lemma:coordinate} which shows that there exists a coordinate $i$ for which there are two subsets of $\QQ$ such that for one of the collection at coordinate $i$ the mean value is less than or equal to $r_i-\gamma$, another collection for which the mean value is more than or equal to $r_i+\gamma$ for some $r_i \in (0,1)$ and the mean vectors of these two collections do not possess a countable $3\gamma-$cover. 
\begin{lemma}
\label{lemma:coordinate} 
For a collection of distributions $\QQ$ that does not possess a countable $3\gamma-$ cover for its mean, there exists a coordinate $i$ and a value $r_i \in (0,1)$ for which there exist two collections $\QQ_1,\QQ_2$ that also do not possess a countable $3\gamma-$cover for their means and if $q \in $ Mean$(\QQ_1)$ then $q_i \ge r_i+\gamma$ and if $q \in $ Mean$(\QQ_2)$ then $q_i \le r_i-\gamma$.
\end{lemma}
\begin{proof}
We define an operation  \emph{subset selection} across a coordinate $i$ with a value $r_i$ for deviation $\gamma$ performed on the collection of distributions $\QQ$ in which we create two collections of distributions $\QQ_{1,i}$ and $\QQ_{2,i}$ such that for any $q \in$ Mean$(\QQ_{1,i}), q_i\le r_i-\gamma$ and for any $q \in $ Mean$ (\QQ_{2,i}), q_i \ge r_i+\gamma$.\\
We first consider the case in which, when we perform \emph{subset selection} for deviation $\gamma$ across all coordinates $i$, there is some value $r_i$ that produces two collections which possess a countable $3\gamma-$cover. Let $\QQ_{1,i},\QQ_{2,i}$ denote the collections obtained after \emph{subset selection} across coordinate $i$ with value $r_i$ for deviation $\gamma$. Let $\bar\QQ_{1,i},\bar \QQ_{2,i}$ denote their respective countable $3\gamma-$covers. Let $\QQ_1 = \cup_{i \in \NN} \QQ_{1,i}$ and $\QQ_2 = \cup_{i \in \NN} \QQ_{2,i}$. We note that $\bar \QQ_1 = \cup_{i \in \NN}\bar \QQ_{1,i}$ is a countable $3\gamma-$cover for Mean$(\QQ_1)$ and $\bar \QQ_2 = \cup_{i \in \NN}\bar \QQ_{2,i}$ is a countable $3\gamma-$cover for Mean$(\QQ_2)$. If $\mu \in \QQ$ does not belong to either $\QQ_1 $ or $\QQ_2$ then due to our \emph{subset selection} operation for every $i \in \NN, \modu{q_i-r_i}<\gamma$ where $q = $ Mean$(\mu)$. Hence $\bar \QQ_1 \cup \bar \QQ_2 \cup \{(r_1,r_2,\ldots)\}$ is a countable $3\gamma-$cover of $\QQ$ contradicting our assumption that Mean$(\QQ)$ does not possess a countable $3\gamma-$cover.\\
We now consider the case in which for all coordinates $i$ with any $r_i \in (0,1)$, at most one of the collections obtained after \emph{subset selection} for deviation $\gamma$ does not possess a countable $3\gamma-$cover. Let the collections obtained after \emph{subset selection} across coordinate $i$ with $r_i = \frac{1}{2}$ for deviation $\gamma$ be $\QQ_{1,i},\QQ_{2,i}$. Without loss of generality, let for all $ i \in \NN,$ Mean$(\QQ_{1,i})$ not possess a countable $3\gamma-$cover whereas Mean$(\QQ_{2,i})$ possess a countable $3\gamma-$cover labeled as $\bar \QQ_{2,i}$. Let $\QQ^1_2 = \cup_{i \in \NN}\QQ_{2,i}$ and $\bar \QQ^1_2 = \cup_{i \in \NN}\bar \QQ_{2,i}$. We note that $\bar \QQ^1_2$ is a countable $3\gamma-$cover of $\QQ^1_2$. We note that for any $\mu \in  \QQ^1_2, q_i \ge \frac12+\gamma$ for every $i \in \NN$ where $q = $ Mean$(\mu)$. Hence for any $\mu \notin \QQ^1_2, q_i<\frac{1}{2}+\gamma$ for every $i \in \NN$ where $q = $ Mean$(\mu)$. \\
We can now similarly repeat the produce of performing \emph{subset selection} on $\QQ^1_1$ for every coordinate $i \in \NN$ with $r_i = \frac{1}{2}\parof{\frac{1}{2}+\gamma}$ and deviation ${\gamma}$ and create $\QQ^2_1$ and $\QQ^2_2$ where Mean$(\QQ^2_1)$ does not possess a countable $3\gamma-$cover and Mean$(\QQ^2_2)$ has a countable $3\gamma-$cover $\bar \QQ^2_2$. \\
We recursively repeat the procedure for $K = \cov{\lceil}{\log_2\parof{\frac{1-2\gamma}{\gamma}}}{\rceil}$ iterations to obtain $K$ countable covers $\bar \QQ^1_2,\bar \QQ^2_2,\ldots,\bar \QQ^K_2$ for Mean$(\QQ^1_2),$ Mean$(\QQ^2_2), \ldots, $ Mean$(\QQ^K_2)$ respectively and $\QQ^K_1$ such that for any $\mu \in \QQ^K_1,q_i<\frac{1}{2^K}+\parof{1+\frac{1}{2}+\ldots+\frac{1}{2^{K-1}}}\gamma \le  3\gamma$ where  $q = $ Mean$(\mu)$. Consequently, Mean$(\QQ^K_1)$ can be covered by $(0,0,\ldots)$. Due to our application of the \emph{subset selection} procedure recursively, we obtain $ \QQ^K_1\cup\bigcup_{k \in [K]}\QQ^k_2  = \QQ$. Hence $\cup_{k \in [K]} \bar \QQ^k_2 \cup (0,0,\ldots)$ is a countable $3\gamma-$cover for Mean$(\QQ)$ which contradiction our assumption. 
\end{proof}
\begin{proof} (Theorem \ref{thm:infseqfatshattered})
We show that Mean$(\QQ)$ is infinitely sequentially fat-shattered by providing an infinite-depth tree such that each branch is realized by some mean vector in Mean$(\QQ)$. As previously argued, as $\QQ$ has non-separable mean vectors, there exists $\gamma>0$ such that there does not exist a countable $3\gamma-$cover. \\
We can build the tree recursively. At the root, Lemma \ref{lemma:coordinate} provides a coordinate $i_1$ that splits $\QQ$ into two collections of distributions that do not possess a countable $3\gamma-$cover. These subsets constitute the collections used to construct the left and right subtrees of the tree. We now consider the left and right subtrees separately. As the collections of distributions do not possess a countable $3\gamma-$cover, we can repeat the previous step. Therefore, we can use Lemma \ref{lemma:coordinate} at every depth of the tree, hence obtaining an infinitely fat shattered tree.
\end{proof}
This inherent structure of infinite fat shattering of the mean vectors of a collection of distributions that have non-separable mean vectors produces an interesting example.
We define a collection of distributions $\QQ_{tree}$ using their respective mean vectors. We consider a binary tree and label it using a mean vector by traversing it in a level order fashion. Formally, for every mean vector $q = (q_1,q_2,\ldots)$, the root corresponds to $q_1$, the left child of $q_1$ corresponds to $q_2$, the right child of $q_1$ corresponds to $q_3$, the left child of $q_2$ corresponds to $q_4$, the right child of $q_2$ corresponds to $q_5$ and so on. As our mean vector is infinite, the binary tree has an infinite depth. Finally, for any branch, i.e., a root-to-leaf path in the tree, we assign all their corresponding coordinates with the value $\frac{2}{3}$, whereas all other coordinates are given a value of $\frac{1}{3}$.\\
Hence, we can define $\QQ_{tree}$ as a collection of distributions for all such mean vectors as follows
\begin{equation*}
    \QQ_{tree} = \curlof{\mu : \mu = \text{Prod}(q) \text{ where } q \text{ satisfies the structure given above}}
\end{equation*}
We note that $\QQ_{tree}$ has non-separable mean vectors as for any $q, q' \in \QQ_{tree}, \norm{q-q'}_\infty = \frac{1}{3}.$ Hence we can use an argument similar to the one used to show the non-separability of Mean$(\QQ_{bin})$.\\
We will now establish the notation for demonstrating that  $\QQ_{tree}$ is UME-learnable.\\
For $\QQ_{tree}$, we will use a tree-specific notation. The root is labeled as $V()$. A branch is identified using a bit string where $0$ indicates a left node and $1$ indicates a right node. At depth $d$, we consider a $d-$dimensional binary vector $(b_1,b_2,\ldots,b_d)$ which provides the root-to-node path. This node is labeled as $V(b_1,b_2,\ldots,b_d)$. For example, for depth $2$, we have the labeling according to Figure 1. \\
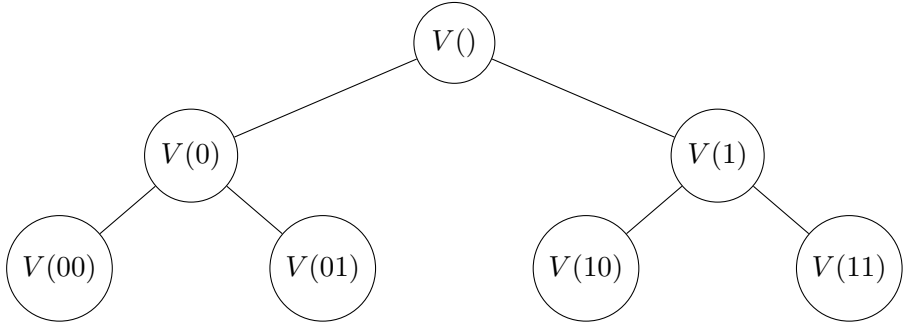
\begin{figure}
\label{fig:onlyfigure}
\centering
\begin{tikzpicture}[level distance=1.5cm,
  level 1/.style={sibling distance=7cm},
  level 2/.style={sibling distance=3.5cm}]
  \node[circle,draw] (root) {$V{()}$}
    child {
      node[circle,draw] (leftchild) {$V{(0)}$}
      child {
        node[circle,draw] (leftleft) {$V{(00)}$}
      }
      child {
        node[circle,draw] (leftright) {$V{(01)}$}
      }  
    }
    child {
      node[circle,draw] (leftchild) {$V{(1)}$}
      child {
        node[circle,draw] (leftleft) {$V{(10)}$}
      }
      child {
        node[circle,draw] (leftright) {$V{(11)}$}
      }  
    };
\end{tikzpicture} 
\caption{Labeling for tree of depth $2$}
\end{figure}
To UME-learn $\QQ_{tree}$, we note that finding the underlying distribution is equivalent to finding the branch labeled with $\frac{2}{3}$. We will refer to this branch as the true branch. We consider the following algorithm,
\begin{algorithm}
        \caption{Tree$(\QQ_{tree},n)$}
        \label{alg:tree}
        $\forall b \in \{0,1\}^\NN$ compute $\phi(b) = \liminf_{d \rightarrow \infty} \frac{1}{d}\sum_{j=1}^d \frac{1}{n}\sum_{i=1}^n X^{(i)}_{V(b_1,\ldots,b_j)}$\\
        Return $\tilde b$ such that $\phi(\tilde b) = \frac23$
\end{algorithm}
\\ The algorithm computes the limiting average of the empirical mean using the $n$ data points available along every branch of the tree. We will show that the value converges to $\frac{2}{3}$ only for the true branch, whereas the value uniformly converges to a value other than $\frac{2}{3}$ for all other branches.
\begin{proposition}
$\QQ_{tree}$ is UME-learnable by Algorithm \ref{alg:tree}.
\end{proposition}
\begin{proof}
We note that due to our construction of $\QQ_{tree}$, finding the true underlying distribution is equivalent to finding the branch that has been labeled as $\frac{2}{3}$. Consequently, we refer to this as the true branch and denote it by $b^*$. Also note that if a parent node is labeled $\frac{1}{3}$, then the child node must be labeled $\frac{1}{3}$ on any branch. This provides us with the core idea of the algorithm. We use this temporal relation to devise the test 
\begin{equation*}
    \label{eqn:treetest}
    \phi(b) = \liminf_{d \rightarrow \infty} \frac{1}{d}\sum_{j=1}^d \frac{1}{n}\sum_{i=1}^n X^{(i)}_{V(b_1,\ldots,b_j)}
\end{equation*}
We note that for the true branch $b^*, \phi(b^*) = \frac{2}{3}$ by the law of large numbers. 
Algorithm \ref{alg:tree} will UME-learn $\QQ_{tree}$ if with probability $1$, for all $b \neq b^*, \phi(b) \neq \frac23$ which can be equivalently proven if
    \begin{equation*}
        \label{eqn:uctree}
        \sup_{b} \cov{|}{\phi(b) - \expect \phi(b)}{|} = \sup_b  \lim_{d \rightarrow \infty}\left|\frac{1}{d} \sum_{j=1}^d \frac{1}{n} \sum_{i=1}^n \parof{X^{(i)}_{V(b_1,\ldots,b_j)} -   q_{V(b_1,\ldots,b_j)}}\right| = 0 
    \end{equation*}
We consider the tree up to depth $d$ and analyze the partial tree. We can apply a union bound over the $2^d$ branches of the partial tree. As these variables are independent (not identically distributed) random variables, we can use Hoeffding inequality (\cite{Hoeffding1963}) to obtain,
\begin{equation}
\label{eqn:convinprob}
    \prob \parof{\sup_b \left|\frac{1}{d} \sum_{j=1}^d \frac{1}{n} \sum_{i=1}^n \parof{X^{(i)}_{V(b_1,\ldots,b_j)} -   q_{V(b_1,\ldots,b_j)}}\right|>\frac{1}{\sqrt{n}}} \le 2^d \cdot 2e^{-2dn\parof{\frac{1}{\sqrt{n}}}^2}<2^{-cd}
\end{equation}
We define event $E_d$ as obtaining a deviation of more than $\frac{1}{\sqrt{n}}$ between the estimated mean of the partial branch at depth $d$ and its true mean. By equation \eqref{eqn:convinprob} we know that $\prob(E_d) <2^{-cd}.$ We note that
$    \sum_{d=0}^\infty \prob(E_d) < \sum_{d=0}^\infty 2^{-cd} = \frac{1}{1-2^{-c}} < \infty$
Hence, by the First Borel-Cantelli lemma, we obtain that  with probability $1, \exists d_0 < \infty$ such that 
\begin{equation}
\label{eqn:d0cond}
    \text{For any } d>d_0,\; \sup_b \left|\frac{1}{d} \sum_{j=1}^d \frac{1}{n} \sum_{i=1}^n \parof{X^{(i)}_{V(b_1,\ldots,b_j)} -   q_{V(b_1,\ldots,b_j)}}\right| < \frac{1}{\sqrt{n}}
\end{equation}
i.e., every sufficiently large depth has deviations that are at most $\frac{1}{\sqrt{n}}$. 
Hence, using equation \eqref{eqn:d0cond} we obtain
\begin{equation*}
\label{eqn:limit}
\lim_{d \rightarrow \infty} \sup_b \left|\frac{1}{d} \sum_{j=1}^d \frac{1}{n} \sum_{i=1}^n \parof{X^{(i)}_{V(b_1,\ldots,b_j)} -   q_{V(b_1,\ldots,b_j)}}\right| < \frac{1}{\sqrt{n}}
\end{equation*}
We can further apply Fatou's Lemma (\cite{rudin1991functional}) to obtain
\begin{equation*}
   \sup_b  \lim_{d \rightarrow \infty}  \left|\frac{1}{d} \sum_{j=1}^d \frac{1}{n} \sum_{i=1}^n \parof{X^{(i)}_{V(b_1,\ldots,b_j)} -   q_{V(b_1,\ldots,b_j)}}\right| < \frac{1}{\sqrt{n}}
\end{equation*}
For our specific example of $\QQ_{tree}$, for all non-true branches, the limiting average of the means along a branch is $\frac{1}{3}$. Therefore, we obtain,
\begin{equation*}
\label{eqn:uniformboundonallbranch}
\text{For any } b \neq b^*, \phi(b)\le \frac{1}{3}+\frac{1}{\sqrt{n}} 
\end{equation*}
So, if $n\ge 36,$
\begin{equation*}
\label{eqn:treespecficbound}
    \text{For any }b \neq b^*, \phi(b)\le \frac{1}{2}
\end{equation*}  
Hence, with probability $1, \text{ for all } b \neq b^*, \phi(b) \neq \frac{2}{3}$ and for the true branch $\phi(b^*) = \frac{2}{3}.$
\end{proof}
\section{UME-learnability for uncountable coordinate sets}
\label{appendix:uncountable}
Our work is motivated by the framework adopted in \cite{cohen2023localglivenkocantelli}, which considers the $P-$Glivenko-Cantelli setting for a countable coordinate set. In section \ref{sec:separable}, we show that collections of distributions that have separable mean vectors are UME-learnable. The technique used in the algorithm to claim UME-learnability (Algorithm \ref{alg:sep}) is to eliminate candidates of the $\eps-$approximation of the mean vector. Whenever a candidate vector deviates excessively from the empirical mean of the first $n$ coordinates, we eliminate it. This approach fundamentally relies on the ability to inspect finitely many coordinates and therefore does not directly extend to uncountable coordinate sets. To overcome this obstacle, we revisit the strategy developed in Section \ref{sec:countableunion}, where we are able to eliminate a candidate estimator using a \emph{single} informative coordinate. We leverage this idea to show that separability of the mean space remains sufficient for UME-learnability even when the coordinate set is uncountable.\\ 
Throughout this section, we assume that all measure-theoretic subtleties can be resolved. We begin by defining an oracle that compares two mean vectors under the $\ell_\infty$ norm. Given two vectors $q^1,q^2$ and a tolerance $\eps$, the oracle either certifies that the vectors are $\eps-$close or returns a coordinate on which they differ by more than $\eps$.
\begin{algorithm}
    \caption{$\ell_\infty$-oracle($q^1,q^2,\eps$)}
    \label{alg:oracle}
    \begin{algorithmic}
    \State \textbf{if} $\norm{q^1-q^2}_\infty <\eps$ \textbf{then }\Return `close'
    \State \textbf{else} \Return $J \in \{j:\modu{q^1_j-q^2_j}>\eps\}$    
    \end{algorithmic}
\end{algorithm}
\\ Using the oracle in Algorithm \ref{alg:oracle}, we modify the survival test as seen in Algorithm \ref{alg:survival}. Given a countable $\eps-$cover of the mean space, we conduct a $1$-vs-$n$ tournament among the first $n$ candidate vectors. A candidate that wins against all others is declared the winner, yielding an $\eps-$approximation of the true mean vector. Crucially, the oracle allows us to select which coordinate is tested, thereby extending UME-learnability to uncountable coordinate sets.
\begin{algorithm}
    \caption{Modified $\eps-$approximate$(\QQ,n,\eps)$}
    \label{alg:modsurvival}
    \begin{algorithmic}
    \State Initialize $\QQ_\eps = \{q^1,q^2,\ldots\}$ as the countable $\eps-$cover of Mean$(\QQ)$
    \State Let $ \hat q$ be the empirical mean computed using the training data.
    \State \textbf{for} {$s$ goes from $1$ to $n$} \textbf{do}
    \State \hspace{\algorithmicindent}Let wins $\gets 0$
    \State \hspace{\algorithmicindent}\textbf{for} {$t$ goes from $1 $ to $n$} \textbf{ do }{
    \State \hspace{\algorithmicindent}\hspace{\algorithmicindent} Let $J=\ell_\infty-$oracle$\parof{q^s,q^t,4\eps}$ 
    \State \hspace{\algorithmicindent}\hspace{\algorithmicindent} \textbf{if} $J$ is `close' \textbf{then} wins $\gets $ wins $+1$
    \State \hspace{\algorithmicindent} \hspace{\algorithmicindent}\textbf{else if}  $\modu{\hat q_J -  q^s_J} < \eps+\sqrt{\frac{3\log n}{n}}$ \textbf{then}  wins $\gets $ wins $+ 1$
    }
    \State \hspace{\algorithmicindent}\textbf{if} wins is equal to $n$ \textbf{then} return $q^s$
    \end{algorithmic}
\end{algorithm}
\begin{lemma}
\label{lemma:modepsappx}
    If collection of distributions $\QQ$ with an uncountable coordinate set has a countable $\eps-$cover for its mean then for any $\mu \in \QQ,$ with probability $1$ there exists a data size $n_0$ such that for all $n>n_0$ the estimator $\tilde q$ returned by Algorithm \ref{alg:modsurvival} satisfies
    \[\norm{\tilde q-q}_\infty\le 5\eps\]
    where $q = $Mean$(\QQ)$
\end{lemma}
\begin{proof}   
    Let a collection of distributions $\QQ$ and $\eps>0$ be given. Let $\QQ_\eps$ be a countable $\eps-$cover of Mean($\QQ$) under the $\ell_\infty$ norm. Let $\mu^*$ be the true underlying distribution and let $q^* =$Mean($\mu^*$). Let $q^{i^*_\eps}$ be vector in $\QQ_\eps$ such that $\norm{q^*-q^{i^*_\eps}}_\infty \le \eps$. We refer to $q^{i^*_\eps}$ as the $\eps-$approximating vector. Let $n>i^*_\eps.$ Let $ I_1 = \{i \in [n]: \norm{q^{i^*_\eps}-q^i}_\infty \le 4\eps\}$ and $I_2 = [n]\setminus I_1$. For any $i \in I_2$ let $j_i \in \{j:\modu{q^{i^*_\eps}_{j}-q^i_{j}}>4\eps\}$. \\
    Let event $E_n$ denote $q^{i^*_\eps}$ failing the tournament against any of the $n$ other vectors. 
    By definition, $q^{i^*_\eps}$ will wins against any vector $q^i$ such that $i \in I_1$. Hence, we focus on winning against $q^i$ such that $ i \in I_2$. To analyze the probability of $E_n$, we use union bound and triangle inequality to obtain
    \begin{equation}
        \label{eqn:uncount1}
        \prob\parof{\exists i \in I_2: \modu{\hat q_{j_i}-q^{i^*_\eps}_{j_i}} >\eps+\sqrt{\frac{3\log n}{n}}} \le \sum_{i=1}^n \prob\parof{\modu{\hat q_{j_i}-q^*_{j_i}}+\modu{q^*_{j_i}-q^{i^*_\eps}_{j_i}} >\eps+\sqrt{\frac{3\log n}{n}}}
    \end{equation}
    We further use the fact that $\norm{q^*-q^{i^*_\eps}}_\infty<\eps$ to obtain
    \begin{equation}
        \sum_{i=1}^n \prob\parof{\modu{\hat q_{j_i}-q^*_{j_i}}+\modu{q^*_{j_i}-q^{i^*_\eps}_{j_i}} >\eps+\sqrt{\frac{3\log n}{n}}} \le \sum_{i=1}^n \prob\parof{\modu{\hat q_{j_i}-q^*_{j_i}}>\sqrt{\frac{3\log n}{n}}}   
    \end{equation}
    Applying Hoeffding's inequality we get, 
    \begin{equation}
        \sum_{i=1}^n \prob\parof{\modu{\hat q_{j_i}-q^*_{j_i}}>\sqrt{\frac{3\log n}{n}}}    \le n \cdot 2e^{-2n\parof{\sqrt{\frac{3\log n}{n}}}^2} = \frac{2}{n^5}.
    \end{equation}
    Let $\tilde q$ denote the vector returned by Algorithm \ref{alg:modsurvival}. A vector $q^i$ such that $i \in I_1$ could also win the tournament. By our previous analysis with probability at least $1-\frac{2}{n^5}$, the index of $\tilde q$ is in $I_1$. But as $\norm{q^i-q^{i^*_\eps}}_\infty<4\eps$, therefore by our previous analysis with probability at least $1-\frac{2}{n^5}$ $\tilde q$ will be $5\eps-$approximation of the true underlying mean vector.\\
    We note that $\sum_{n=1}^\infty \prob(E_n)\le \sum_{n=1}^\infty \frac{2}{n^5}<\infty$. Hence, we can apply the First Borel-Cantelli Lemma to conclude that with probability $1$ there exists $n_0>0$ such that for all $n>n_0$ the algorithm successfully finds a $5\eps-$approximating vector.
    \end{proof}
    We now modify Algorithm \ref{alg:sep} by using Algorithm \ref{alg:modsurvival} instead of Algorithm \ref{alg:epsapprox}, thereby extending UME-learnability to a collection of distributions indexed by an uncountable set.
     \begin{algorithm}
 \begin{algorithmic}
\caption{Modified Separable $(\mathcal{Q}, n>0)$}
\label{alg:modsep}
\State Initialize $\PP \gets \text{Mean}(\QQ)$ where $\text{Mean}(\QQ)$ is as in equation \eqref{eqn:meanvector} \\
$\tilde q \gets \emptyset, k \gets 1 $ 
\State \textbf{while} {$\PP$ is not empty} \textbf{do}{
    \State \hspace{\algorithmicindent} $\eps_k \gets \frac{1}{2^k}, \tilde q \gets $ any $q \in \PP$ 
    \State \hspace{\algorithmicindent} Run Algorithm \ref{alg:modsurvival}$(\QQ,n,\eps_k)$ to obtain $q^{k}$
    \State \hspace{\algorithmicindent} $\PP \gets \PP \cap \mathcal{B}(q^{k},\eps_k)$
    \State \hspace{\algorithmicindent} $k \gets k+1$
}
\State \Return{$\tilde q$}\;
\end{algorithmic}
\end{algorithm}
    \begin{theorem}
    \label{thm:modseparable}
        If the collection of distributions with an uncountable coordinate set $\QQ$ has separable mean vectors, then $\QQ$ is UME-learnable by Algorithm \ref{alg:modsep}.
    \end{theorem}
    The proof is similar to the proof of Theorem \ref{thm:separable}. We use Lemma \ref{lemma:modepsappx} instead of Lemma \ref{lemma:epsapprox}.
\section{Uniform UME-learnability}
\label{app:uniform-learnability}
In our work, we focus on uniform convergence over a function class and not over the collection of distributions, and we analyze UME-learnability asymptotically. In this section, we show that if the mean vectors of a collection of distributions are totally bounded, then we can provide non-asymptotic bounds on the expected loss using algorithm \ref{alg:totallybounded}. We say the mean vectors of a collection of distributions $\QQ$ are totally bounded if for every $\eps>0$ there exists a \emph{finite} $\eps-$cover for Mean$(\QQ)$. \\
\begin{algorithm}
    \caption{Totally Bounded $\eps-$approximate$(\QQ_{TB},n)$}
    \label{alg:totallybounded}
    \begin{algorithmic}
    \State Let $N $ be the $\eps-$covering number for Mean$(\QQ_{TB})$ under the $\ell_\infty$ norm.
    \State Let $\QQ_\eps = \{q^1,q^2,\ldots,q^N\}$ as the countable $\eps-$cover of Mean$(\QQ_{TB})$
    \State Let $ \hat q$ be the empirical mean computed using the training data.
    \State \textbf{for} {$s$ goes from $1$ to $N$} \textbf{do}
    \State \hspace{\algorithmicindent}Let wins $\gets 0$
    \State \hspace{\algorithmicindent}\textbf{for} {$t$ goes from $1 $ to $N$} \textbf{ do }
    \State \hspace{\algorithmicindent}\hspace{\algorithmicindent} \textbf{if} for every $j \in \NN \modu{ q^s_j -  q^t_j} \le 4 \eps$ \textbf{then} wins $\gets $ wins $+1$
    \State \hspace{\algorithmicindent} \hspace{\algorithmicindent}\textbf{else}
    \State \hspace{\algorithmicindent} \hspace{\algorithmicindent} \hspace{\algorithmicindent} $J = \min\{j \in \NN: \modu{q^s_j -  q^t_j} > 4 \eps\}$
    \State \hspace{\algorithmicindent} \hspace{\algorithmicindent} \hspace{\algorithmicindent} \textbf{if} $\modu{\hat q_J -  q^s_J} < 2\eps$ \textbf{then}  wins $\gets $ wins $+ 1$
    \State \hspace{\algorithmicindent}\textbf{if} wins is equal to $n$ \textbf{then} return $q^s$ 
    \end{algorithmic}
\end{algorithm}\\
Algorithm \ref{alg:totallybounded} is a modification of Algorithm \ref{alg:survival} in which for every $\eps>0$ as the $\eps-$cover is \emph{finite} we can find a $5\eps-$approximation of the true underlying distribution with probability at least $1-2\eps$ by comparing all the vectors against each other after obtaining a sufficiently large amount of data points. 
\begin{theorem}
    Let $\QQ_{TB}$ be a collection of distributions such that Mean$(\QQ_{TB})$ is totally bounded. Let $N(\eps)$ denote the $\eps-$covering number of Mean$(\QQ_{TB})$. $\QQ_{TB}$ is UME-learnable using Algorithm \ref{alg:totallybounded} such that for every $\mu \in \QQ_{TB}$,  \[\expect_{S\sim{\mu^n}} \norm{\tilde q-q}_\infty \le  7 \inf_{\eps>0} \curlof{\eps:n>\frac{1}{2\eps^2}\log\parof{\frac{N(\eps)}{\eps}}}\]
    where $q = $Mean$(\QQ)$ and $\tilde q = $ Totally Bounded $\eps-$approximate$(\QQ_{TB},n)$
\end{theorem}
\begin{proof}
    Let $\QQ_{TB}$ be a collection of distributions such that Mean$(\QQ_{TB})$ is totally bounded. Let $\QQ_\eps$ be a finite $\eps-$cover of Mean($\QQ_{TB}$) under the $\ell_\infty$ norm. Let $N$ denote the $\eps-$covering number of Mean($\QQ_{TB}$). Let $\mu^*$ be the true underlying distribution and let $q^* =$Mean($\mu^*$). 
    Let $q^{i^*_\eps}$ be vector in $\QQ_\eps$ such that $\norm{q^*-q^{i^*_\eps}}_\infty \le \eps$. We refer to $q^{i^*_\eps}$ as the $\eps-$approximating vector. \\
    Let $n>\frac{1}{2\eps^2}\log \parof{\frac{N}{\eps}}$. 
    Let $ I_1 = \{i \in [N]: \norm{q^{i^*_\eps}-q^i}_\infty \le 4\eps\}$ and $I_2 = [N]\setminus I_1$. For any $i \in I_2$ let $j_i = \min \{j\in \NN:\modu{q^{i^*_\eps}_{j}-q^i_{j}}>4\eps\}$. \\
    We analyze the probability that $q^{i^*_\eps}$ loses a comparison against some $q^i \in \QQ_\eps$. By using union bound and triangle inequality, we obtain
    \begin{equation}
        \prob\parof{\exists i \in I_2: \modu{\hat q_{j_i}-q^{i^*_\eps}_{j_i}} >2\eps} \le \sum_{i=1}^N \prob\parof{\modu{\hat q_{j_i}-q^*_{j_i}}+\modu{q^*_{j_i}-q^{i^*_\eps}_{j_i}} >2\eps}
    \end{equation}
    We further use the fact that $\norm{q^*-q^{i^*_\eps}}_\infty<\eps$ to obtain
    \begin{equation}
        \sum_{i=1}^N \prob\parof{\modu{\hat q_{j_i}-q^*_{j_i}}+\modu{q^*_{j_i}-q^{i^*_\eps}_{j_i}} >2\eps} \le \sum_{i=1}^N \prob\parof{\modu{\hat q_{j_i}-q^*_{j_i}}>\eps}   
    \end{equation}
    We further apply Hoeffding inequality (\cite{Hoeffding1963}),
    \begin{equation}
        \sum_{i=1}^N \prob\parof{\modu{\hat q_{j_i}-q^*_{j_i}}>\eps}    \le N \cdot 2e^{-2n\eps^2} \le 2N e^{-2\frac{1}{2\eps^2}\log \parof{\frac{N}{\eps}}\eps^2} = 2\eps
    \end{equation}
    Let $\tilde q$ be the vector returned by running Algorithm \ref{alg:totallybounded} on $\QQ_{TB}$ using sufficiently large amount of training data $\parof{\text{i.e. }n \ge \frac{1}{2\eps^2}\log\parof{\frac{N}{\eps}}}$. We note that with probability at least $1-2\eps$, any vector in $I_1$ could have been returned by the algorithm. Thus, the algorithm will return a $5\eps-$approximation of the true underlying mean vector with probability of error at most $2\eps$.\\
    Therefore, we note that 
    \[\expect\norm{\tilde q-q}_\infty \le 5\eps \cdot \prob\parof{\norm{\tilde q-q}_\infty \le 5\eps}+1\cdot \prob(\norm{\tilde q-q}_\infty>5\eps) \le 7\eps\]
    Therefore if we have been provided with $n$ data points, we can optimize for $\eps$ to obtain \[\expect \norm{\tilde q - q}_\infty\le 7 \inf_{\eps>0} \curlof{\eps:n>\frac{1}{2\eps^2}\log\parof{\frac{N(\eps)}{\eps}}}\]
\end{proof}
\end{document}